\documentclass{article}

\usepackage[preprint,nonatbib]{neurips_2020}

\usepackage[utf8]{inputenc} 
\usepackage[T1]{fontenc}    
\usepackage{hyperref}       
\usepackage{url}            
\usepackage{booktabs}       
\usepackage{nicefrac}       
\usepackage{microtype}      
\usepackage[english]{babel}

\usepackage[utf8]{inputenc}
\usepackage{amssymb,amsfonts,amsmath}
\usepackage{amsthm}
\usepackage{comment}
\usepackage{bm}
\usepackage{mathtools}
\usepackage{breqn}
\usepackage{float}
\usepackage{algorithmic}
\usepackage{algorithm, setspace}
\usepackage{subcaption}

\newtheorem{definition}{Definition}
\newtheorem{theorem}{Theorem}

\title{STORM: Foundations of End-to-End Empirical Risk Minimization on the Edge}

\author{
  Benjamin Coleman\thanks{Equal contribution.} \\
  Electrical and Computer Engineering\\
  Rice University\\
  Houston, Texas, USA \\
  \texttt{ben.coleman@rice.edu} \\
  \And
  Gaurav Gupta\textsuperscript{*} \\
  Electrical and Computer Engineering\\
  Rice University\\
  Houston, Texas, USA \\
  \texttt{gaurav.gupta@rice.edu} \\
  \And 
  John Chen\\
  Department of Computer Science\\
  Rice University\\
  Houston, Texas, USA \\
  \texttt{jc114@rice.edu} \\
  \And 
  Anshumali Shrivastava\\
  Department of Computer Science\\
  Rice University\\
  Houston, Texas, USA \\
  \texttt{anshumali@rice.edu} \\
}

\begin{document}

\maketitle

\begin{abstract}
Empirical risk minimization is perhaps the most influential idea in statistical learning, with applications to nearly all scientific and technical domains in the form of regression and classification models. To analyze massive streaming datasets in distributed computing environments, practitioners increasingly prefer to deploy regression models on edge rather than in the cloud. By keeping data on edge devices, we minimize the energy, communication, and data security risk associated with the model. Although it is equally advantageous to train models at the edge, a common assumption is that the model was originally trained in the cloud, since training typically requires substantial computation and memory. To this end, we propose STORM, an online sketch for empirical risk minimization. STORM compresses a data stream into a tiny array of integer counters. This sketch is sufficient to estimate a variety of surrogate losses over the original dataset. We provide rigorous theoretical analysis and show that STORM can estimate a carefully chosen surrogate loss for the least-squares objective. In an exhaustive experimental comparison for linear regression models on real-world datasets, we find that STORM allows accurate regression models to be trained. 
\end{abstract}

\section{Introduction}

Global spending on the Internet of Things (IoT) is expected to reach 1 trillion US dollars by 2021, according to a recent spending report~\cite{IoTSpendingGuide}. Machine Learning (ML) for IoT is a pressing issue from both an economic perspective and an algorithm development perspective. Modern ML algorithms are expensive in terms of latency, memory and energy. While these resource demands are not limiting in a cloud computing environment, they pose serious problems for ML on mobile and edge devices. The current trend in IoT research is to minimize the cost, memory and energy associated with inference. Since training is substantially more costly than inference, models are usually trained in the cloud before being modified and deployed to edge devices. 

However, there are distinct advantages to training on the edge. Since data is generally collected from edge and mobile devices, the data is inevitably transmitted to the cloud or some other compute-intensive source for training. Even with compression, data transfer is an energy intensive procedure that consumes substantial resources. Data transfer also introduces privacy vulnerabilities. In~\cite{IoTnetwork}, the authors argue that it is significantly more efficient to do computations on the device, rather than transmitting data out of the edge. The computation on mobile and edge devices has grown significantly, but memory is limited and large data transfer is prohibitive due to energy costs.

\paragraph{Foundations of Training ML Models on the Edge:} In this paper, we build the foundation of the feasibility of training machine learning models directly on edge devices. There are few existing techniques because training at the edge involves several challenges. First, the data must be compressed to an extremely small size so that it can comfortably fit on the device. Due to resource limitations, it is impossible to explicitly store the entire dataset. The compression technique must be one-pass and computationally efficient. The compressed summary, or sketch, of the data must be updateable with new information collected by the device. This tiny sketch should be sufficient to train the ML model. We also anticipate that the sketches must be combined across multiple devices to train a larger, more powerful model. One can imagine a scenario where IoT devices propagate their sketches along the edges of a communication network, updating their models and passing the information forward. This type of system requires a \emph{mergeable} summary~\cite{rabkin2014aggregation}. 

\paragraph{Formal Setting - Sketches for Regression:}
We focus on compressed and distributed empirical risk minimization. Given an $N$-point dataset $\mathcal{D} = \{(x_1,y_1) ... (x_N,y_N)\}$ of $d$-dimensional examples observed in the streaming setting~\cite{fiat1998online}, the goal is to construct a small structure $\mathcal{S}$ that can estimate the regression loss over $\mathcal{D}$. The size of $\mathcal{S}$ must scale well with the number of features and the size of $\mathcal{D}$. Ideally, the sketch should occupy only a few megabytes (MB) of memory and be small enough to transmit over a network. $\mathcal{S}$ should also be a mergeable summary~\cite{agarwal2013mergeable}. If we construct a sketch $\mathcal{S}_1$ of a dataset $\mathcal{D}_1$ and $\mathcal{S}_2$ of $\mathcal{D}_2$, we should be able to merge $\mathcal{S}_1$ with $\mathcal{S}_2$ to get a sketch of the combined dataset $\mathcal{D}_1\cup \mathcal{D}_2$. These properties ensure that the sketch is useful in distributed streaming settings.

\subsection{Related Work}

Regression models are perhaps the most well-known embodiment of empirical risk minimization, having been independently studied by the machine learning, statistics, and computer science communities. Existing work on compressed regression is largely motivated by prohibitive runtime costs for large, high-dimensional linear models and prohibitive storage costs for large data matrices. The canonical least-squares solution requires computation and storage that scales quadratically with dimensions and linearly with $N$, leading to a proliferation of approximate solutions. Approximations generally fall into two categories: sketching and sampling. One may either apply linear projections to reduce the dimensionality of the problem or apply sampling to reduce $N$. Regardless of the approach taken, there is a well-known information theoretic lower bound on the space required to solve the problem within an $\epsilon$ approximation to the optimal loss~\cite{clarkson2009numerical}. 

Sketch-based approaches to online compressed regression approximate the large $N\times d$ data matrix with a smaller representation. Such algorithms are agnostic about the data in that they rely on subspace embeddings rather than structure within the data matrix to provide a reasonable size reduction. In~\cite{clarkson2009numerical}, the authors propose a sketch that attains the memory lower bound for compressed linear regression. The sketch stores a random linear projection that can recover an approximation to the regression model. Random Gaussian, Hadamard, Haar, and Bernoulli sign matrices been used in this framework with varying degrees of theoretical and practical success~\cite{dobriban2019asymptotics}. While these sketches can only estimate the L2 loss, the ``sketch-and-solve'' technique has been extended to regression with the L1 objective~\cite{sohler2011subspace} and several other linear problems. 

Strategies based on sampling are attractive in that they can dramatically reduce the size $N$ of the data that is used with linear algebra routines. Sampling methods can also be applied to more loss functions than sketching methods. The simplest and fastest approach is random sampling. Unfortunately, random sampling has undesirable worst-case performance as it can easily miss important samples that contribute substantially to the model. As a result, adaptive sampling procedures based on leverage score sampling have been developed for the regression problem. Leverage scores can be approximated online ~\cite{cohen2016online}, but are somewhat computationally expensive in practice.

\paragraph{Our Contribution:}
In this work, we propose a sketch with a large set of desirable practical properties that can perform empirical risk minimization with a small memory footprint. The count-based nature of our Sketches Toward Online Risk Minimization (STORM) enables optimizations that are not possible for other methods. Specifically, our contributions are as follows:
\begin{itemize}
    \item We propose an online sketching algorithm (STORM) that can estimate a class of loss functions using only integer count values. Our sketch is appropriate for distributed data settings because it is small, embarrassingly parallel, mergeable via addition, and able to work with streaming data. 
    \item We characterize the set of loss functions that can be approximated with STORM. The price of efficient loss estimation is a restriction on the kinds of losses we can use. While the class of STORM-approximable losses does not include popular regression and classification losses, we derive STORM-approximable surrogate losses for linear regression and linear max margin classification. 
    \item We show how to perform optimization over STORM sketches using derivative free optimization and linear optimization. We provide experiments with the linear regression objective showing that STORM-approximable surrogates are a resource-efficient way to train models in distributed settings. 
\end{itemize}

\section{Background}
\subsection{Empirical Risk Minimization}

In the standard statistical learning framework, we are given a training dataset $\mathcal{D} = \{(x_i,y_i) \in \mathcal{D}\times \mathcal{Y}\}^{n}_{i = 1}$ of examples and asked to select a function $h:\mathcal{X}\to\mathcal{Y}$ that can predict $y$ given $x$. In this paper, we will present our algorithms for the data space $\mathcal{X} = \mathbb{R}^d$ and output space $\mathcal{Y} = \mathbb{R}$, but our results also hold for other metric spaces. The learning problem is select a hypothesis $h_{\theta}$ (a function parameterized by $\theta$) that yields good predictions, as measured by a loss function $\mathrm{loss}(h_{\theta}(x),y)$. In empirical risk minimization (ERM), we select $\theta$ to minimize the average loss on the training set.

$$ \hat{\theta} = \operatorname*{arg\,min}_\theta \sum_{i = 1}^n \mathrm{loss}(h_{\theta}(x_i),y_i)$$

\paragraph{Linear Regression:} 
We will often use linear regression as an example. Linear regression is an embodiment of empirical risk minimization where $h_{\theta}(\mathbf{x}) = \langle \theta, \mathbf{x}\rangle$. For most applications, $\theta$ is found using the least-squares or L2 loss: $\mathrm{loss}(h_{\theta}(\mathbf{x})) = ||h_{\theta}(\mathbf{x}) - y||_2^2$. The unconstrained L2 loss is smooth and strongly convex, with desirable convergence criteria. The parameter $\theta$ is found using ERM, either using gradient descent or via a closed-form solution from the matrix formulation of the model $\textbf{y} = \textbf{X} \theta$. For our discussion, it will be important to express the loss in terms of the concatenated vector $[\mathbf{x}_i,y_i]$: 
$$\hat{\theta} = \operatorname*{arg\,min}_\theta \sum_{i=1}^{n} (\langle [\mathbf{x_i},y_i], [\theta,-1] \rangle)^2 = \operatorname*{arg\,min}_\theta \| \textbf{y} -\textbf{X} \theta\|_2^2$$

\subsection{Locality Sensitive Hashing}
\label{sec:lsh}
Locality sensitive hashing (LSH) is a technique from computational geometry originally introduced for efficient approximate nearest neighbor search. An LSH family $\mathcal{F}$ is a family of functions $l(\mathbf{x}): \mathcal{X}\to \mathbb{Z}$ with the following property: Under $l(\mathbf{x})$, similar points have a higher probability than dissimilar points of having the same hash value, or \textit{colliding} (i.e. $l(\mathbf{x}_1) = l(\mathbf{x}_2)$).  The notion of similarity is usually based on the distance measure of the metric space $\mathcal{X}$. For instance, there are LSH families for the Jaccard~\cite{broder1997minhash}, Euclidean~\cite{datar2004locality,dasgupta2011fast} and angular distances~\cite{charikar2002similarity}. If we allow asymmetric hash constructions (i.e. $l_1(\mathbf{x_1} = l_2(\mathbf{x}_2)$), $\mathbf{x}_1$ and $\mathbf{x}_2$ may collide based on other properties such as their inner product~\cite{shrivastava2014mips}. To accommodate asymmetric LSH, we use the following definition of LSH. Our definition is a strict generalization of the original~\cite{indyk1998approximate}, which can be recovered by setting probability thresholds.

\begin{definition}
\label{def:lsh}
We say that a hash family $\mathcal{F}$ is locality-sensitive with collision probability $k(\cdot,\cdot)$ if for any two points $\mathbf{x}$ and $\mathbf{y}$ in $\mathbb{R}^d$, $l(\mathbf{x}) = l(\mathbf{y})$ with probability $k(\mathbf{x},\mathbf{y})$
under a uniform random selection of $l(\cdot)$ from $\mathcal{F}$. 
\end{definition}

One example of a symmetric LSH family is the signed random projection (SRP) family for the angular distance~\cite{goemans1995improved, charikar2002similarity}. The SRP family is the set of functions $l(\mathbf{x}) = \mathrm{sign}(\mathbf{w}^{\top}\mathbf{x})$, where $\mathbf{w}\sim\mathcal{N}(0,I_d)$. The SRP collision probability is 

$$k(\mathbf{x},\mathbf{y}) = 1-\frac{1}{\pi} \cos ^{-1}\left(\frac{\langle \mathbf{x},\mathbf{y} \rangle}{\|\mathbf{x}\|_{2}\|\mathbf{y}\|_{2}}\right)$$

The inner product hash~\cite{shrivastava2014mips} is an example of an asymmetric LSH. Suppose we replace $\mathbf{x}$ with $\left[\mathbf{x}, 0, \sqrt{1- \|\mathbf{x}\|_2^2}\right]$ and $\mathbf{y}$ with $\left[\mathbf{y}, \sqrt{1- \|\mathbf{y}\|_2^2}, 0 \right]$ in the SRP function. This procedure essentially uses different hash functions for $\mathbf{x}$ and $\mathbf{y}$, but the collision probability $k(\mathbf{x},\mathbf{y})$ is now a monotone function of the inner product $\langle\mathbf{x},\mathbf{y}\rangle$. In particular, $k(\mathbf{x},\mathbf{y})$ is the same as for SRP but without the normalization by $\|\mathbf{x}\|$ and $\|\mathbf{y}\|$. Here, we implicitly assume that $\mathbf{x}$ and $\mathbf{y}$ are inside the unit sphere, so we often scale the dataset when using this inner product hash in practice.

\paragraph{Collision Probabilities and Sketching:} The collision probability $k(x,y)$ is a positive function of $x$ and $y$. For most LSH functions, $k(x,y)$ is a function of the distance $d(x,y)$ and is a positive definite kernel function~\cite{coleman2020race}. For asymmetric LSH functions, $k(x,y)$ can take on nontrivial shapes because $k(x,y)\neq k(y,x)$.
The RACE streaming algorithm~\cite{luo2018ace,coleman2020race} provides an efficient sketch to estimate the sum 
$$ L(\theta) = \frac{1}{n}\sum_{i = 1}^{n} k(\theta,x_i)$$
when $k(\cdot,\cdot)$ forms the collision probability of an LSH function. The sketch consists of a sparse 2D array of integers that are indexed with an LSH function. RACE sketches have error guarantees that relate the quality of the approximation to the memory and computation requirements. The error depends mostly on the value of $L(\theta)$ and does not directly depend on $N$. 

The RACE sketch can also be released with differential privacy without substantially increasing the error bound~\cite{coleman2020private}. There are two directions toward a private sketch: private LSH functions and private RACE sketches. Private LSH functions can be constructed by adding Gaussian noise to the hash function value for projection-based hash functions~\cite{Kenthapadi_Korolova_Mironov_Mishra_2013}. This strategy preserves $(\epsilon,\delta)$-differential privacy for the data attributes of each example $x$ in the dataset. Private RACE sketches can be constructed by adding Laplace noise to the count values for each cell in the sketch. This strategy preserves $\epsilon$-differential privacy at the example-level granularity. 

In practice, RACE can estimate $L(\theta)$ within 1\% error in a 4 MB sketch~\cite{coleman2020race} when $k(\theta,x)$ is a positive definite kernel. To construct a RACE sketch, we create an empty integer array with $R$ rows and $B$ columns. We construct $R$ LSH functions $\{l_1(x)...l_R(x)\}$ with collision probability $k(\theta,x)$. We increment column $l_r(x)$ in row $r$ when an element $x$ arrives from the stream. We estimate the loss $L(\theta)$ by returning the average count value at the $r$ indices $[r,l_r(\theta)]$.

\section{STORM Sketches for Estimating Surrogate Losses}
While collision probabilities have been well-studied in the context of improving near-neighbor search~\cite{gionis1999similarity}, we study them from a new perspective. Our goal is to compose useful loss functions from LSH collision probabilities and perform ERM on sketches. So far, RACE sketches have enabled new applications in metagenomics~\cite{coleman2019diversified} and compressed near neighbor search~\cite{coleman2020neighbor} because they can efficiently represent the kernel density estimate. RACE sketches reduce the computation and memory footprint because they replace complicated sampling and indexing procedures with a simple sketch~\cite{coleman2020race}. 

We propose RACE-style sketches to approximate the empirical risk for distributed ERM problems. In this context, we query the sketch with the parameter $\theta$ to estimate the empirical risk. This requires some extensions to the RACE sketch. First, we use asymmetric LSH functions to hash data $x$ and $\theta$ with different functions. We also propose methods that apply several LSH functions to $x$ and increment more than one index in each row, adding together collision probability functions. For instance, we propose Paired Random Projections (PRP) for an LSH surrogate linear regression loss. PRP inserts elements into the sketch with two signed random projection (SRP) hashes but queries $\theta$ using only one SRP. Our algorithm to generate Sketches Toward Online Risk Minimization (STORM) is presented in Figure~\ref{fig:stormfig} and Algorithm~\ref{alg:STORMsketch}.

\label{STORMalgo}

\begin{figure}[ht]
\begin{minipage}{0.4\textwidth}
    \centering
    \includegraphics[width=2.7in]{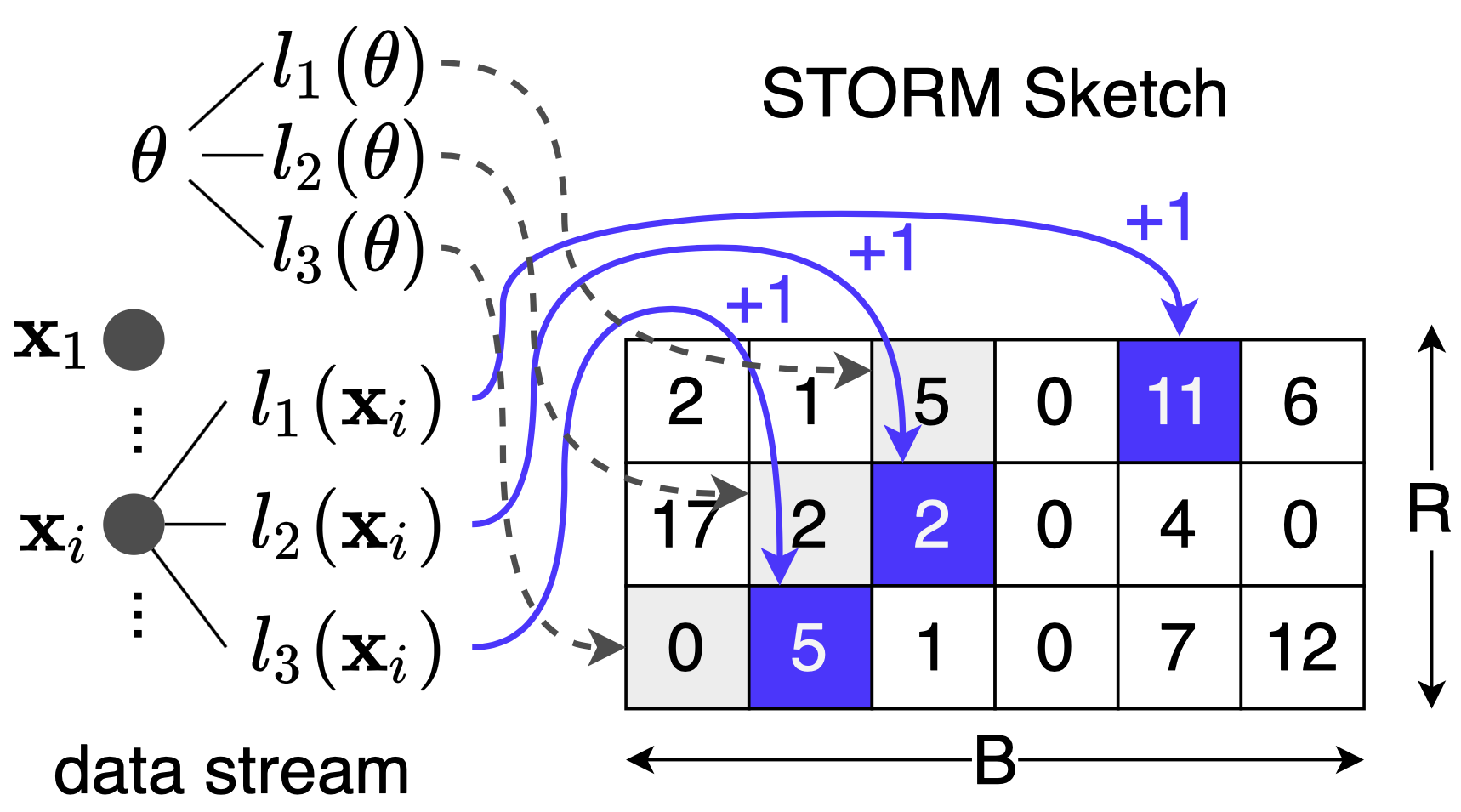}
\end{minipage}
\hfill
\begin{minipage}{0.45\textwidth}
\centering
\begin{algorithm}[H]
\begin{algorithmic}
\STATE {\bf Input:} Dataset $\mathcal{D}$
\STATE {\bf Result:} STORM sketch $\mathcal{S} \in \mathbb{Z}^{R \times B}$
\FOR{$x \in \mathcal{D}$}
\FOR{$r = 1 ... R$}
    \STATE Increment $\mathcal{S}_{r}[ l_r(x)]$
\ENDFOR
\ENDFOR
 \end{algorithmic}
  \caption{STORM Sketch}
\label{alg:STORMsketch}
\end{algorithm}
\end{minipage}
\caption{To construct a STORM sketch, we initialize $R$ arrays with $B$ buckets (or columns)
. We initialise $R$ arrays of size $d$ each. $h_i(.)$ is the $i_{th}$ LSH function.
STORM sketching process. $h_i(.)$ is PRP. Right: shows the algorithm for sketch creation. }
\label{fig:stormfig}
\end{figure}

\paragraph{Intuition:} To understand why STORM sketches contain enough information to perform ERM, consider the task presented in Figure~\ref{fig:intuition}. By partitioning the space into $B$ random regions, we can find the approximate location of the input data by examining the overlap of densely-populated partitions. We can then identify a good regression model as one that mostly passes through dense regions. We can also solve classification problems by favoring hyperplanes that separate dense regions with different classes. STORM sketches extend this idea by optimizing the model $\theta$ over the counts to find a $\theta$ that collides with many data points. The main challenge is to design an LSH function $l(\theta)$ that has a large (or small) count when $\theta$ is a good model.

\begin{figure*}[t]
\centering
\includegraphics[width=\textwidth,keepaspectratio]{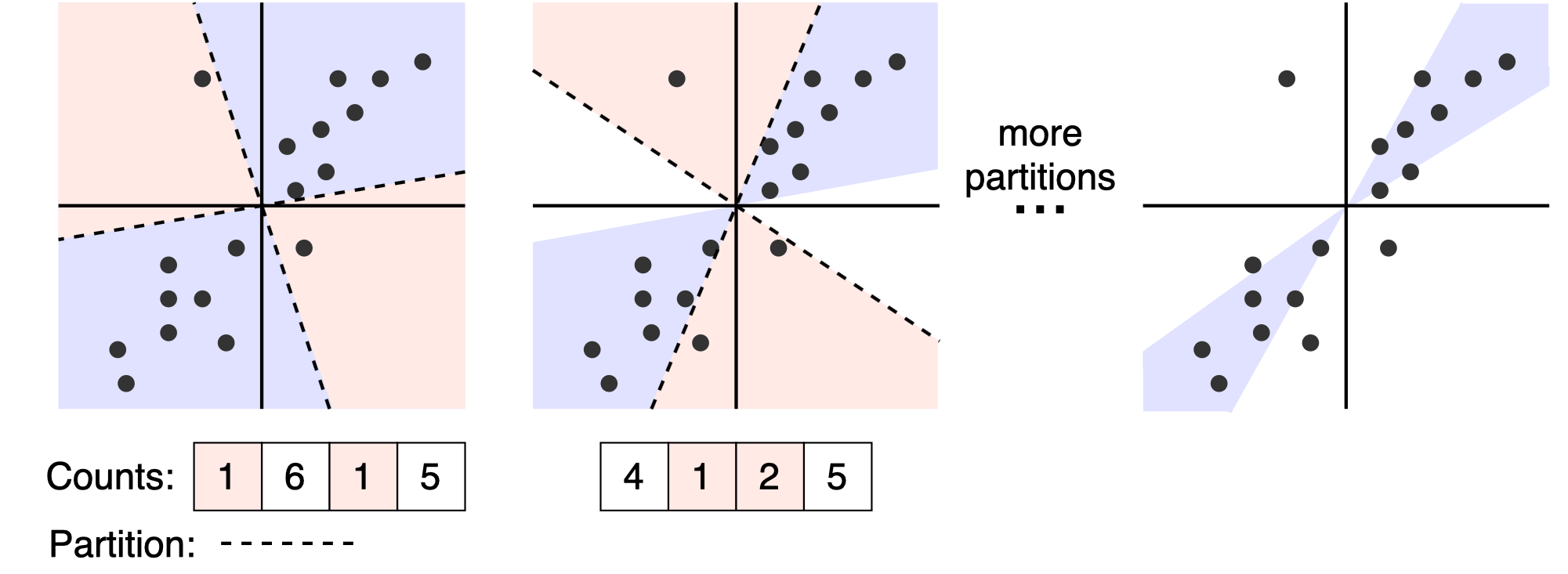}
\vspace{-0.5cm}
\caption{LSH counts are sufficient for learning algorithms. For linear regression, we can narrow down the set of eligible regression lines by eliminating those that live in sparsely-populated partitions (shown in red). After a few random partitions, we are left with a good idea of the optimal solution. We gain the most information when the partition boundaries are orthogonal, which is nearly always true in high dimensions.}
\label{fig:intuition}
\vspace{-0.2cm}
\end{figure*}

\paragraph{Optimization:} Once we have a sketch and an appropriate hash function, we want to optimize the model parameter $\theta$ to minimize the STORM estimate of the empirical risk. The standard ERM technique is to apply gradient descent to find the minimum. Due to the count-based nature of the sketch, we cannot analytically find the gradient. We will resort to derivative-free optimization techniques~\cite{conn2009introduction}, where black-box access to the loss function of interest (or its sharp approximation) is sufficient. Since our focus is not on derivative-free optimization, we employ a simple optimization algorithm that queries the sketch at random points in a sphere around $\theta$. Using only a few ($\sim$10) cheap loss evaluations, we approximate the gradient and update $\theta$.

For some STORM sketches, we obtain improvements over standard derivative-free methods with linear optimization. Such methods attempt to place $\theta$ into the optimal set of hash partitions. Linear optimization is possible when the hash function is a projection-based LSH in $\mathbb{R}^d$.

\section{Theory: Sketchable Surrogate Loss with Same Minima}
\label{STORMtheory}

We begin with a formal discussion of the families of losses that STORM can approximate. Using the compositional properties of LSH, we can show that STORM can provide an unbiased estimator for a large class of functions. For instance, we can insert elements to the STORM sketch (or use a second sketch) to add collision probabilities. We can also concatenate LSH functions to estimate the product of collision probabilities. Theorem~\ref{thm:stormLoss} describes the set of functions that STORM can approximate. 
\begin{theorem}
\label{thm:stormLoss}
The set of STORM-approximate functions $S_L$ contains all LSH collision probabilities and is closed under addition, subtraction, and multiplication.
\end{theorem}
Theorem~\ref{thm:stormLoss} says that we can approximate any sum and/or product of LSH collision probabilities using one or more STORM sketches. Given the number of known LSH families, the flexibility of asymmetric LSH, and the closure of $S_L$ under addition and multiplication, Theorem~\ref{thm:stormLoss} suggests that $S_L$ is an expressive space of functions. In particular, we show that $S_L$ contains useful losses for machine learning problems. 

\subsection{Constructing STORM Approximable Surrogate Loss for Linear Regression}
Since we query the STORM sketch with $\theta$, we must be able to express $\mathrm{loss}(h_{\theta}(x),y)$ as $\mathrm{loss}(\theta,[x,y])$. This condition is not terribly limiting, but it does restrict our attention to hypothesis classes where $\theta$ directly interacts with $x$ and $y$ in a way that can be captured by an LSH function. Linear regression is the simplest example of such a model.

\paragraph{Designing the Loss:}In the non-regularized linear regression objective, $\theta$ interacts with the data via the inner product $\langle [\theta,-1],[\mathbf{x},y]\rangle$. The linear regression loss function is a monotone function of the absolute value of this inner product. The asymmetric hash function discussed in Section~\ref{sec:lsh} seems promising, as it depends on the inner product between $\theta$ and $\mathbf{x}$. However, its collision probability is monotone in the inner product, not the absolute value. To obtain a surrogate loss with the correct dependence on $\langle [\theta,-1],[\mathbf{x},y]\rangle$, observe that the collision probability is monotone decreasing if we hash $-\mathbf{x}$ instead of $\mathbf{x}$. Thus, we can obtain a function that is monotone in $|\langle [\theta,-1],[\mathbf{x},y]\rangle|$ by adding together the collision probability for $\mathbf{x}$ (which is monotone increasing) and for $-\mathbf{x}$ (which is monotone decreasing). We refer to our construction as paired random projections (PRP) because the method can be implemented by hashing $[\mathbf{x},y]$ and $-[\mathbf{x},y]$ with the same SRP function. Rather than update a single location in the sketch, we update the pair of random projection locations. When we query with the vector $\tilde{\theta} = [\theta,-1]$, STORM estimates the following surrogate loss for linear regression. Here, $p$ is an integer power of 2 which determines the number of random projections used for the PRP hash function. 

$$ g(\tilde{\theta},[\mathbf{x},y]) = \frac{1}{2}\left(1 - \frac{1}{\pi}\cos^{-1}({\langle \tilde{\theta},[\mathbf{x},y]\rangle})\right)^{p} + \frac{1}{2}\left(1 - \frac{1}{\pi}\cos^{-1}(- {\langle \tilde{\theta},[\mathbf{x},y]\rangle})\right)^{p} $$

\begin{theorem} When $p\geq 2$, the PRP collision probability $g([\theta,-1],[\mathbf{x},y])$is a convex surrogate loss for the linear regression objective such that 
$$ \operatorname*{arg\,min}_\theta \sum_{\mathbf{x},y\in\mathcal{D}} g([\theta,-1],[\mathbf{x},y]) =  \operatorname*{arg\,min}_\theta \|\mathbf{y} - \mathbf{X}\theta\| $$

\end{theorem}

The integer $p$ is a parameter that determines the number of hyperplanes used by the PRP hash functions. The $p$ hyperplanes split the data space into $2^p$ partitions. Intuition suggests that $p$ should be large, because we can obtain more information when we have many partitions. However, one can observe from Figure~\ref{fig:lossvsL2}(a) that the surrogate loss landscape becomes very flat near the optimum for large $p$, making the function hard to optimize.  We find that $p = 4$ results in the most strongly convex function in a localized region around the optimum. This can be visualized by examining the gradient, for example, at $\langle \theta, y[\mathbf{x},-1] \rangle = 0.1$. Near the optimal regression line, the steepness of the loss basin varies with $p$ as shown in the figure \ref{fig:lossvsL2} (b). This is of practical importance because fewer noisy estimates of $g$ are necessary to optimize a strongly convex function with derivative-free optimization methods.

\begin{figure}
\begin{subfigure}{.45\textwidth}
    \centering
    \includegraphics[scale =0.4]{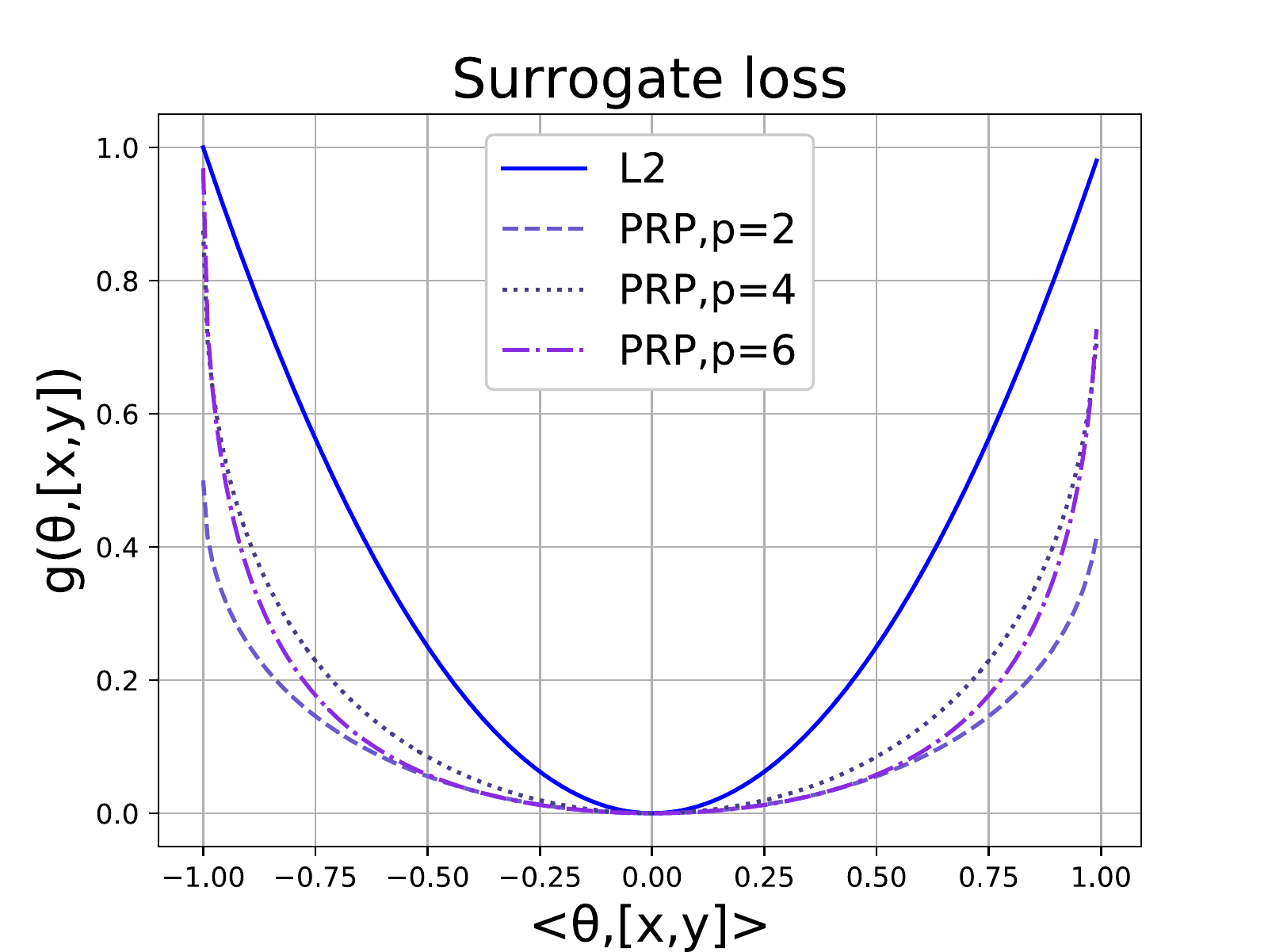}
\end{subfigure}
\hfill
\begin{subfigure}{.45\textwidth}
    \centering
    \includegraphics[scale =0.4]{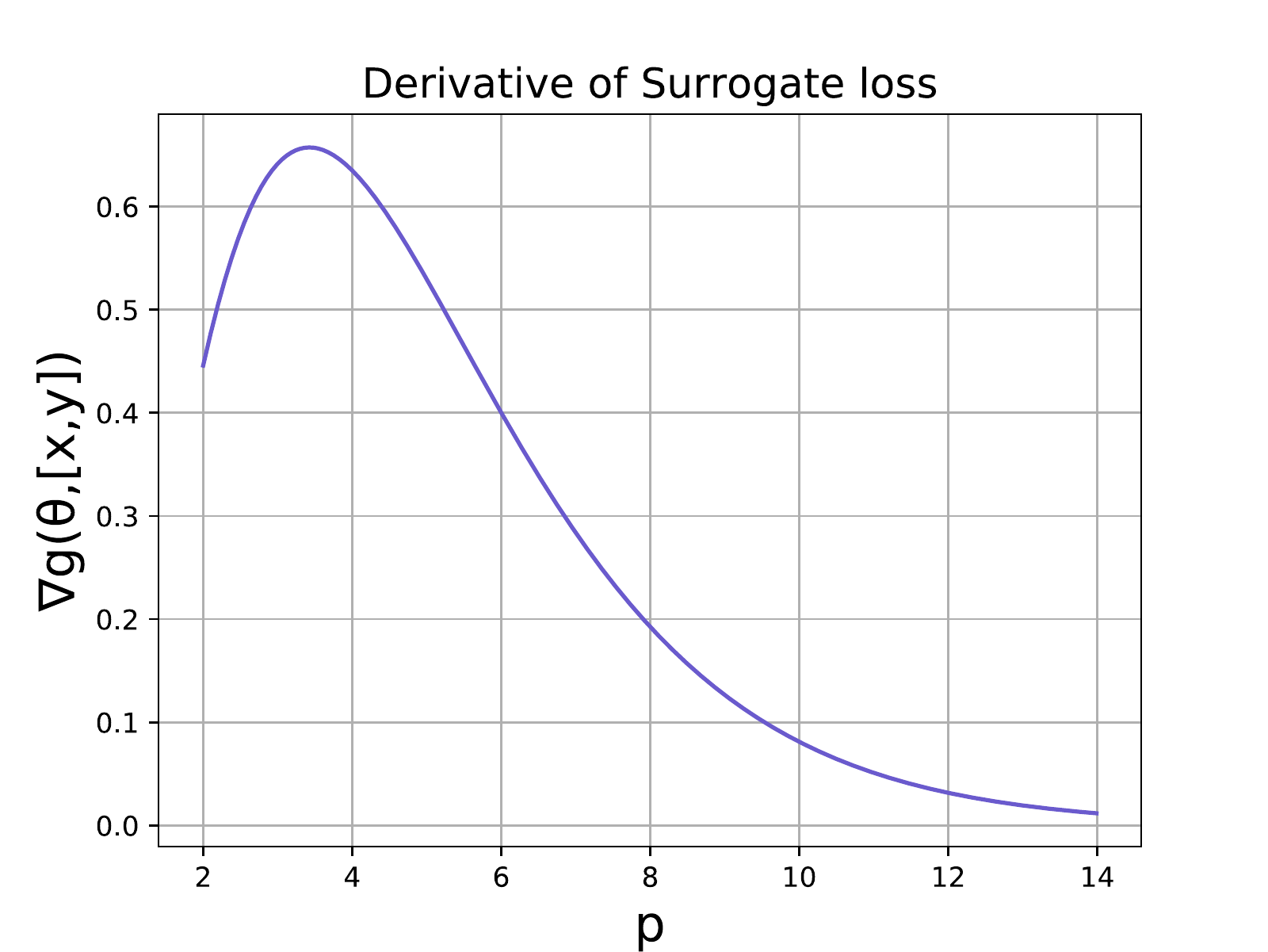}
\end{subfigure}
\caption{(a) Plot of surrogate loss for different values of $p$. The surrogate loss is convex, but less sharp than the L2 loss. A key observation is that $p=4$ produces the highest convexity among the family of PRP surrogate losses. (b) Plot of slope surrogate loss for different values of $p$ at $\langle \theta, y[\mathbf{x},-1] \rangle = 0.1$. The surrogate loss has the steepest slope (and is thus easiest to optimize) near $p = 4$.}
\label{fig:lossvsL2}
\end{figure}

Once the STORM sketch with $R$ repetitions and $d =2^p$ bins is created, we proceed with optimization. We perform derivative-free optimization as discussed before, with an additional constraint. We initialize $\tilde{\theta}$ to zeros in $d+1$ dimensions. The additional dimension is due to querying the sketch with $[\theta,-1]$ rather than $[\theta]$. We compute the approximate gradient by querying the sketch with equidistant points in a ball around $\tilde{\theta}$ and update the parameter. After each iteration, we project the last dimension of $\tilde{\theta}$ back onto the constraint $\tilde{\theta}_{d+1} = -1$. Please refer to Algorithm \ref{alg:gradientFree} for details.

\begin{algorithm}[t]
\begin{algorithmic}
\STATE {\bf Input:} STORM sketch $\mathcal{S} \in \mathbb{Z}^{R\times B}$ constructed with PRP, number of queries $k$ to approximate the gradient, step size $\eta$ and sphere size $\sigma$)
\STATE {\bf Result:} Linear regression parameter $\theta$
\STATE {\bf Result:} Initialise $\tilde{\theta}_0 \leftarrow \mathbf{0}^{d+1}$
\FOR{number of iterations $n = 1, ... N$ }
    \STATE $\{v_1, v_2, .. v_k\}$ = random vectors on the $\sigma$-sphere centered at $\tilde{\theta}$
    \STATE $\{\delta_1, \delta_2, .. \delta_k\} = \{\mathrm{Query}_{\mathcal{S}}(v_1), ... \mathrm{Query}_{\mathcal{S}}(v_k)\}$
    \STATE Approximate gradient $\hat{\nabla}(\tilde{\theta}) = \mathrm{mean}(\delta_1, \delta_2, .. \delta_k)$
    \STATE Update $\tilde{\theta}_{n} = \tilde{\theta}_{n-1} - \eta(g(\tilde{\theta}))$
    \STATE Project last dimension of $\tilde{\theta}$ onto the constraint $-1$
\ENDFOR 

 \end{algorithmic}
  \caption{Derivative-free linear regression with STORM}
  \label{alg:gradientFree}
\end{algorithm}

\subsection{Constructing STORM Approximable Surrogate Loss for Max-Margin Linear Classification}

Analogous to linear regression, we can solve other risk minimization problems using STORM. Consider linear hyperplane classifiers of the form $h_{\theta}(x) = \mathrm{sign}(\langle \theta, [\mathbf{x},-1]\rangle)$. Most popular losses used to find $\theta$ are classification-calibrated margin losses. A loss function is classification-calibrated, or Bayes consistent, if the optimal hypothesis under the loss is the same as the Bayes optimal hypothesis~\cite{bartlett2006convexity}. We propose the following classification-calibrated loss function. 
\begin{theorem}
The loss function $g(\theta,[\mathbf{x},y])$ for the linear hyperplane classifier is a classification-calibrated margin loss and is STORM-approximable.
$$ g(\theta,[\mathbf{x},y]) = 2^p\left(1 - \frac{1}{\pi} \cos^{-1}(-y\langle \theta, \mathbf{x} \rangle) \right)^{p}$$
\end{theorem}
By optimizing $\theta$ in a similar fashion as with regression, we can train linear classifiers with STORM sketches. The LSH function that implements $g(\tilde{\theta},[\mathbf{x},y])$ is the asymmetric inner product hash, but with the argument to the hash function multiplied by $y\in \{-1,+1\}$. See the appendix for details.

\section{Experiments}
\label{experiments}

\begin{table*}[t]\caption{UCI datasets used for linear regression experiments Each dataset has $N$ entries with $d$ features. For all datasets, we run Algorithm~\ref{alg:gradientFree} with $\sigma = 0.5$ and $k = 8$ derivative-free gradient components.}
  \centering
  \begin{tabular}{ l|c|c|l  } 
\toprule
Dataset & $N$ & $d$ & Description \\
\midrule
airfoil & 1.4k & 9 &  Airfoil parameters to predict sound level \\
autos & 159 & 26 &  Automobile prices and information to predict acquisition risk\\
parkinsons & 5.8k & 21 & Telemonitoring data from parkinsons patients, with disease progression \\
\bottomrule
\end{tabular}

\label{tab:datasets}
\vspace{-0.2cm}
\end{table*}

\paragraph{Datasets:}We performed experiments on three UCI datasets, described in Table~\ref{tab:datasets}. We selected datasets with different dimensions, characteristics and sizes. We mainly consider higher-dimensional regression problems ($d\sim 20$), though we provide qualitative results on simulated 2D regression data to provide intuition about the type of regression models found by STORM.

\paragraph{Baselines:} We compare our method against sampling baselines and sketch-based methods i.e. random sampling, leverage score sampling, and the linear algebra sketch proposed by Clarkson and Woodruff~\cite{clarkson2009numerical} for compressed linear regression. We implement all baselines using the smallest standard data type and compare against a range of parameters. 

\paragraph{Experiment setup:}For our regression sketches on STORM, we use PRP with $p =4$ to create the sketch and we vary the number of repetitions. We report results on the training risk, since our objective is to show that the parameters found using STORM are also minimizers of the empirical risk function. We average over 10 runs for our baselines and sketches, where each run has an independently-constructed sketch or random sample. Thus, our average is over the random LSH functions used to construct the sketch and the stochastic derivative-free gradient descent instances.

\paragraph{Results:} In Figure~\ref{fig:regressionPlots}, we report the mean square error for our method when compared with baselines at a variety of memory budgets. We observe a double descent phenomenon for our sampling baselines, explaining the peak near the intrinsic dimensionality of the problem. This sample-wise double descent behavior was recently proved for linear regression by~\cite{nakkiran2019more}. STORM does not experience the double descent curve in practice because the entire dataset (not just a subsample) is used to minimize the loss. We perform favorably against baselines in memory regimes affected by double descent and STORM performs competitively in other memory regimes. We also observe that the $\theta$ found using STORM converges to the optimal $\theta$ under least-squares ERM. This validates our theory that PRP provides a surrogate loss.  


\begin{figure}[H]
    \centering
\begin{subfigure}{0.32\textwidth}
    \includegraphics[width=1.8in]{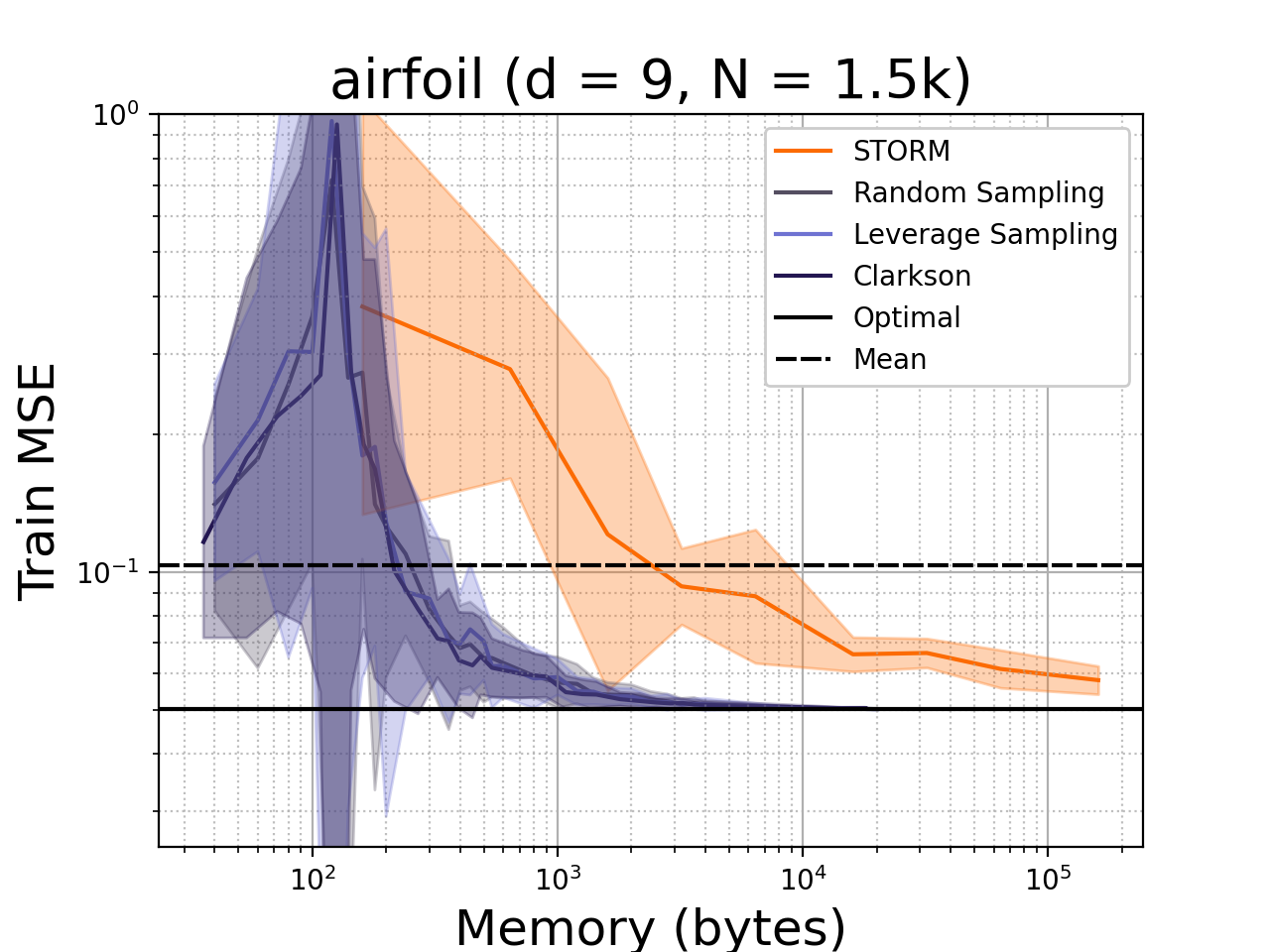}
\end{subfigure}
\begin{subfigure}{0.32\textwidth}
    \includegraphics[width=1.8in]{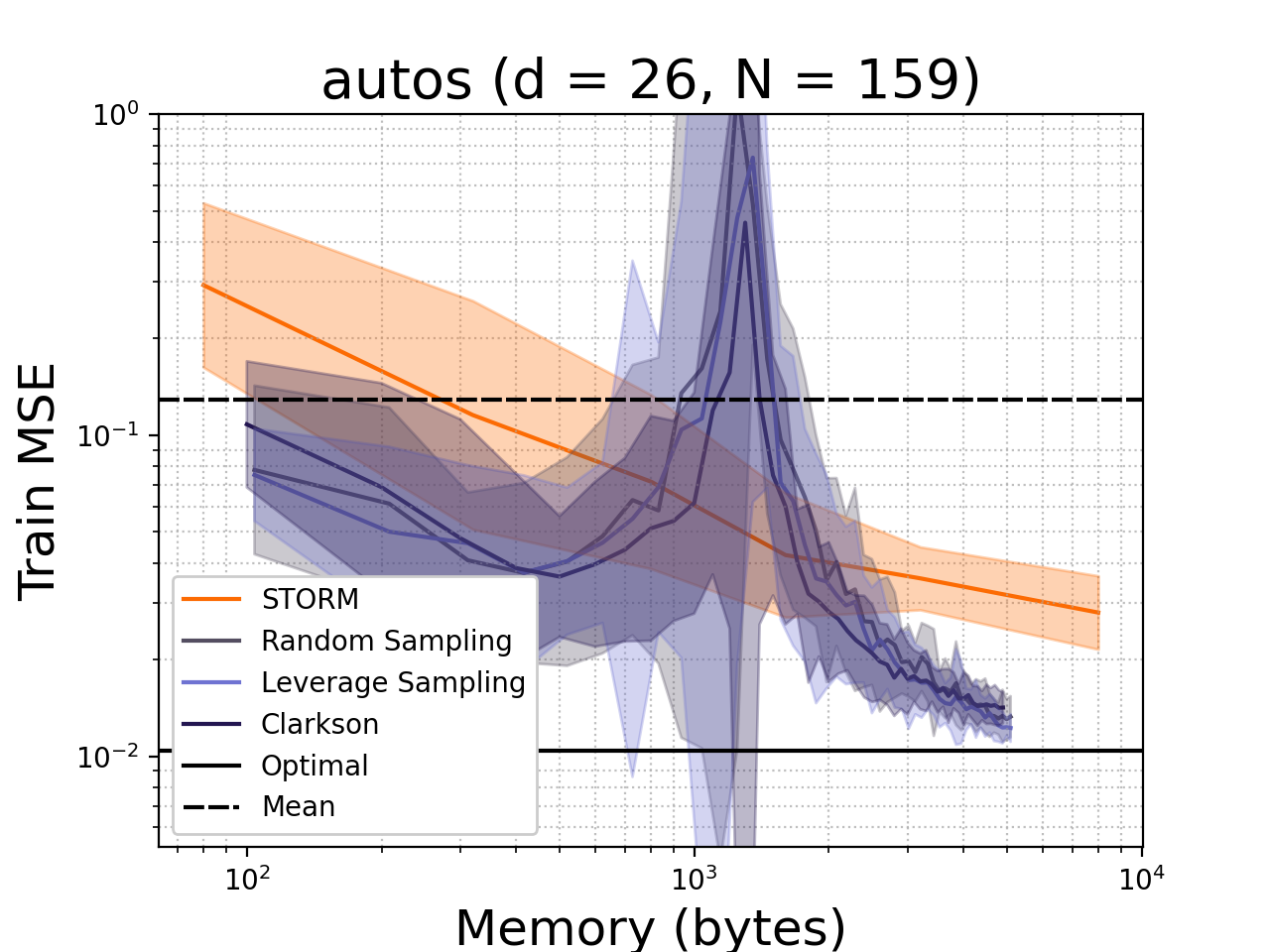}
\end{subfigure}
\begin{subfigure}{0.32\textwidth}
    \includegraphics[width=1.8in]{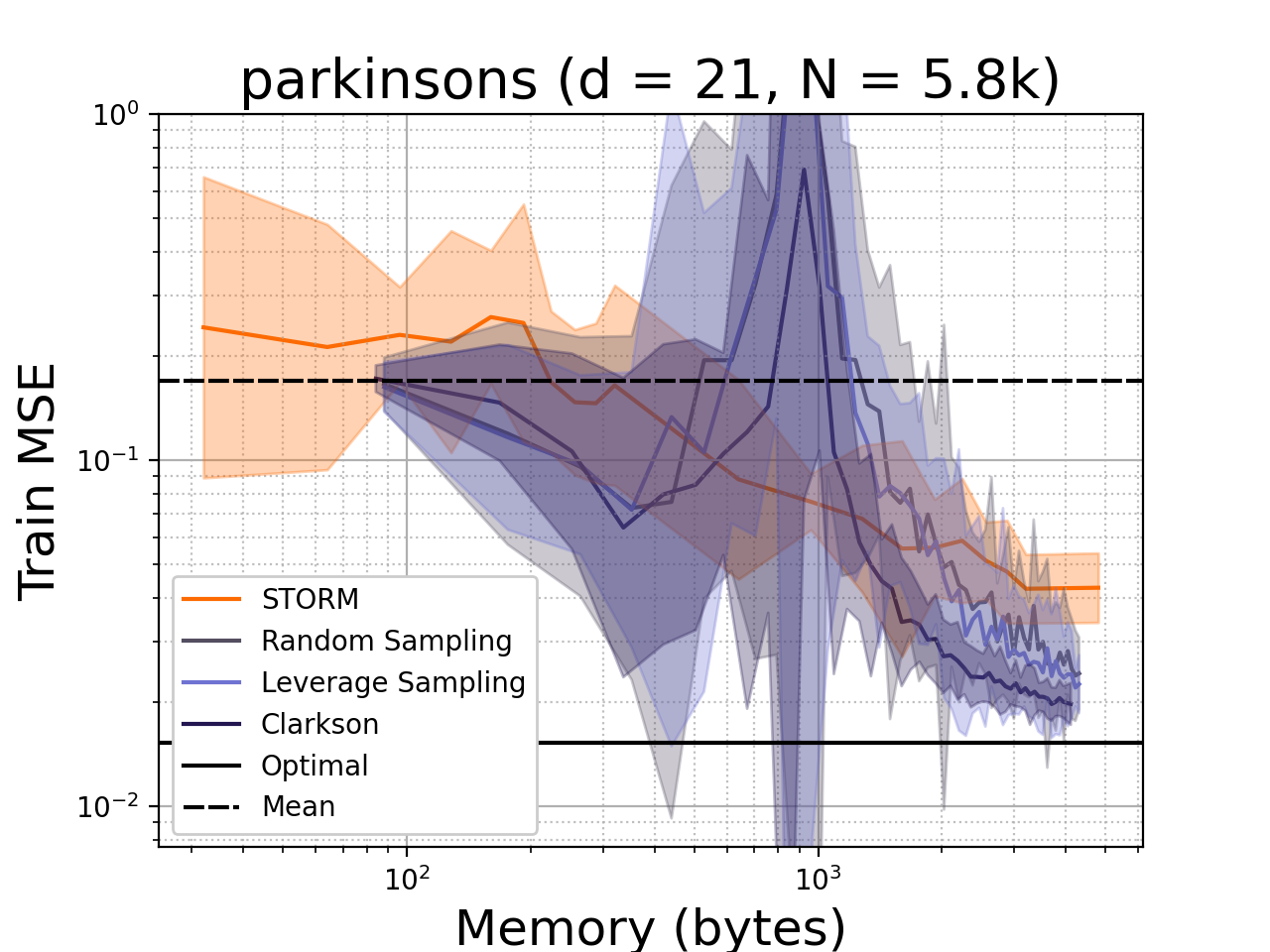}
\end{subfigure}
    \caption{Sketch size vs linear regression models. STORM sketches compare favorably with our sampling and sketch baselines. }
    \label{fig:regressionPlots}
\end{figure}

We also evaluate the linear regression and classification STORM losses on 2D synthetic data (Figure~\ref{fig:synthetic}). We generated synthetic datasets and ran the derivative-free optimizer on a STORM sketch for 100 iterations. We used $R = 100$ for both experiments, with $p = 4$ for regression and $p = 1$ for the classification loss. 

\begin{figure}[H]
    \centering
    \includegraphics[height=1.6in]{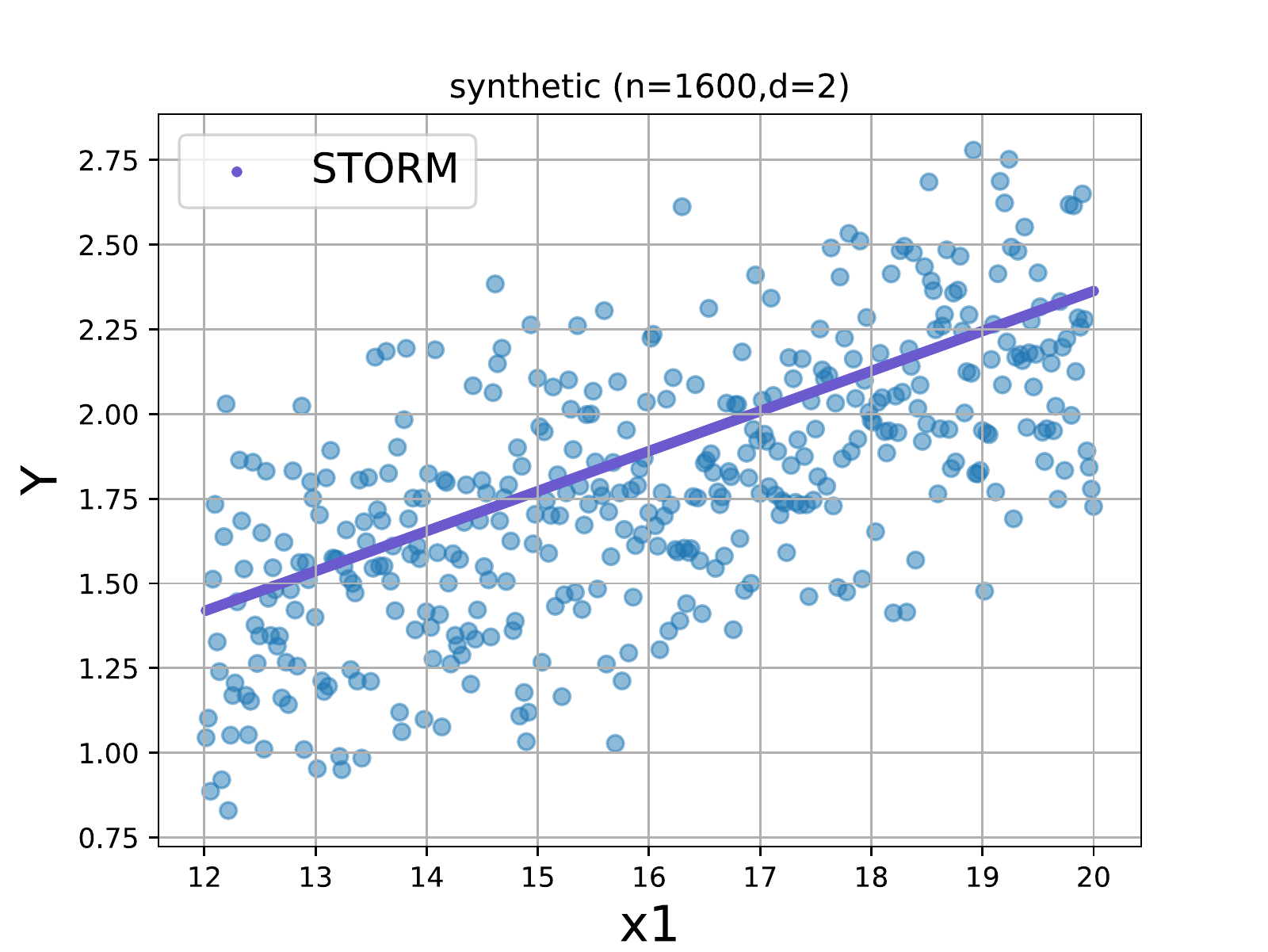}
    \includegraphics[height=1.6in]{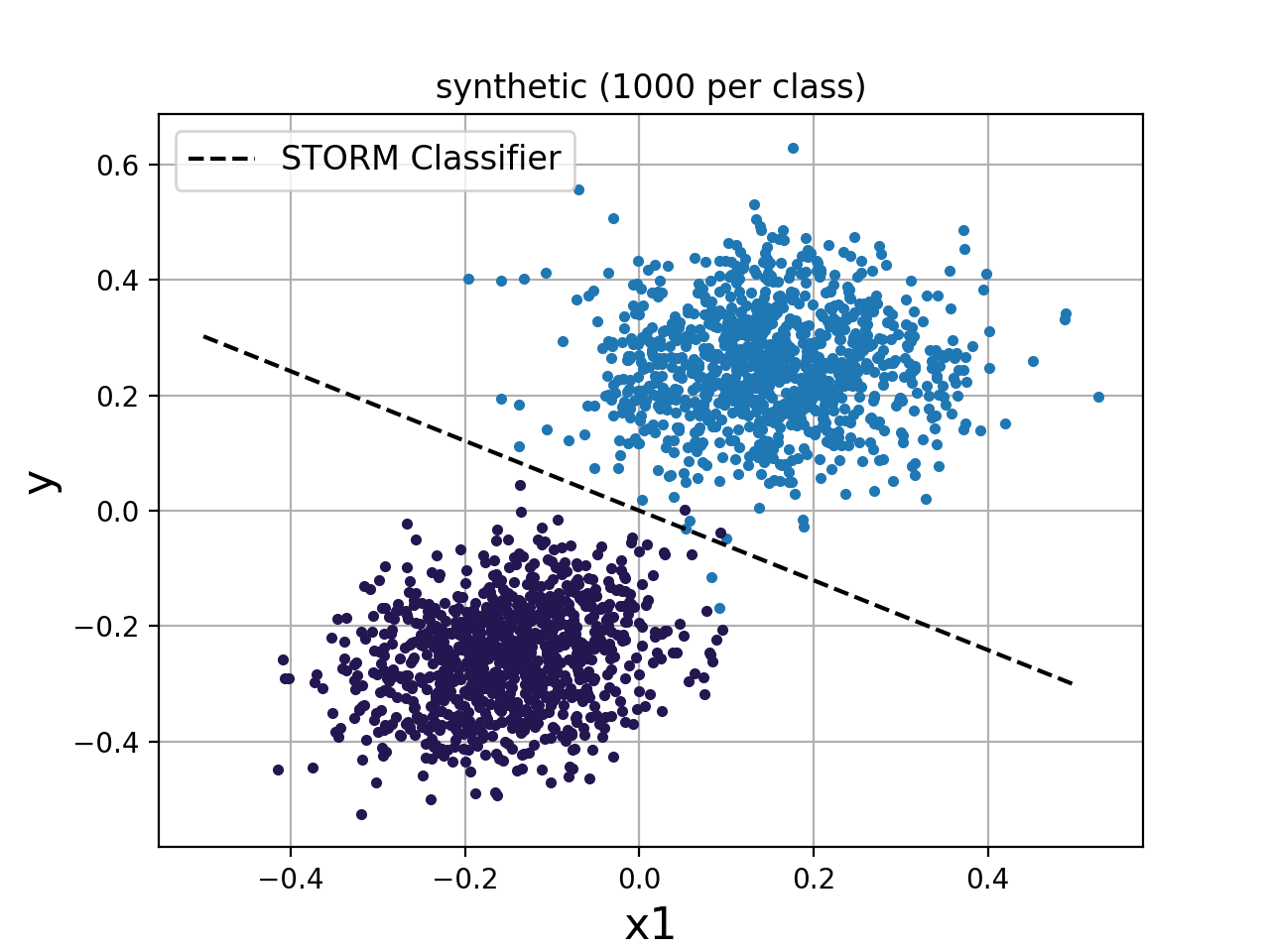}
    \caption{Regression and classification STORM losses on synthetic 2D datasets. }
    \label{fig:synthetic}
\end{figure}

\section{Discussion}

STORM is a scalable way to solve ERM problems, with an efficient streaming implementation and quality risk approximation. While STORM provides a pointwise and non-smooth approximation to the empirical risk, we find that few estimators (small $R$) are necessary to obtain a model parameter $\theta$ similar to the one obtained by minimizing the full loss summation. Our sketch-based loss estimators are sufficiently sharp to compete with sampling and linear algebra baselines, while also naturally accommodating regularization, streaming settings and many useful surrogate losses. Given the simplicity and scalability of the sketch, we expect that STORM will enable distributed learning via ERM on edge devices.

\newpage
\section{Broader Impacts}
The prospect of retaining and training useful models solely on the edge eliminates the opportunity for data interception or leakage, critical to keeping data secure and private. There are significantly reduced privacy concerns if no data is transmitted. This is key to the application and widespread adoption of machine learning in many industries, particularly in industries such as in healthcare where is there is a large demand for machine learning but similarly great data security concerns, particularly with the moral and significant financial ramifications. 
An important implication is the prospective elimination of central data repositories. The danger of major data leaks is curbed without a central repository.
Furthermore, data transmission is one of the most energy consuming tasks devices on the edge will undergo, especially if the data is in the form of video or audio. Due to the proliferation of edge devices, reducing the data transmission energy consumption of edge devices is one of the most pressing problems in machine learning. STORM allows models to be trained at the edge with minimal storage, eliminating the privacy and energy consumption of data transmission and storage.

\bibliographystyle{plain}
\bibliography{references}

\begin{thebibliography}{10}

\bibitem{IoTnetwork}
{Why edge computing is critical for the IoT}.
\newblock NetworkWorld,
  \url{https://www.networkworld.com/article/3234708/why-edge-computing-is-critical-for-the-iot.html}.

\bibitem{IoTSpendingGuide}
{Worldwide Internet of Things Spending Guide}.
\newblock \url{https://www.idc.com/getdoc.jsp?containerId=IDC_P29475}.
\newblock Accessed: 2010-09-30.

\bibitem{agarwal2013mergeable}
Pankaj~K Agarwal, Graham Cormode, Zengfeng Huang, Jeff~M Phillips, Zhewei Wei,
  and Ke~Yi.
\newblock Mergeable summaries.
\newblock {\em ACM Transactions on Database Systems (TODS)}, 38(4):1--28, 2013.

\bibitem{bartlett2006convexity}
Peter~L Bartlett, Michael~I Jordan, and Jon~D McAuliffe.
\newblock Convexity, classification, and risk bounds.
\newblock {\em Journal of the American Statistical Association},
  101(473):138--156, 2006.

\bibitem{broder1997minhash}
Andrei Broder.
\newblock On the resemblance and containment of documents.
\newblock {\em IEEE Compression and Complexity of Sequences}, 1997.

\bibitem{charikar2002similarity}
Moses Charikar.
\newblock Similarity estimation techniques from rounding algorithms.
\newblock {\em STOC}, 2002.

\bibitem{clarkson2009numerical}
Kenneth~L Clarkson and David~P Woodruff.
\newblock Numerical linear algebra in the streaming model.
\newblock In {\em Proceedings of the forty-first annual ACM symposium on Theory
  of computing}, pages 205--214, 2009.

\bibitem{cohen2016online}
Michael~B Cohen, Cameron Musco, and Jakub Pachocki.
\newblock Online row sampling.
\newblock {\em Approximation, Randomization, and Combinatorial Optimization.
  Algorithms and Techniques}, 2016.

\bibitem{coleman2020neighbor}
Benjamin Coleman, Richard~G Baraniuk, and Anshumali Shrivastava.
\newblock Sub-linear memory sketches for near neighbor search on streaming data
  with race.
\newblock In {\em International Conference on Machine Learning}, 2020.

\bibitem{coleman2019diversified}
Benjamin Coleman, Benito Geordie, Li~Chou, RA~Leo Elworth, Todd~J Treangen, and
  Anshumali Shrivastava.
\newblock Diversified race sampling on data streams applied to metagenomic
  sequence analysis.
\newblock {\em bioRxiv}, page 852889, 2019.

\bibitem{coleman2020private}
Benjamin Coleman and Anshumali Shrivastava.
\newblock A one-pass private sketch for most machine learning tasks.
\newblock {\em arXiv preprint arXiv:2006.09352}, 2020.

\bibitem{coleman2020race}
Benjamin Coleman and Anshumali Shrivastava.
\newblock Sub-linear race sketches for approximate kernel density estimation on
  streaming data.
\newblock In {\em Proceedings of the 2020 World Wide Web Conference}.
  International World Wide Web Conferences Steering Committee, 2020.

\bibitem{conn2009introduction}
Andrew~R Conn, Katya Scheinberg, and Luis~N Vicente.
\newblock {\em Introduction to derivative-free optimization}, volume~8.
\newblock Siam, 2009.

\bibitem{dasgupta2011fast}
Anirban Dasgupta, Ravi Kumar, and Tamas Sarlos.
\newblock Fast locality-sensitive hashing.
\newblock {\em KDD}, 2011.

\bibitem{datar2004locality}
Mayur Datar, Nicole Immorlica, Piotr Indyk, and Vahab Mirrokni.
\newblock Locality sensitive hashing scheme based on p-stable distributions.
\newblock {\em Symposium on Computational Geometry}, 2004.

\bibitem{dobriban2019asymptotics}
Edgar Dobriban and Sifan Liu.
\newblock Asymptotics for sketching in least squares regression.
\newblock In {\em Advances in Neural Information Processing Systems}, pages
  3670--3680, 2019.

\bibitem{fiat1998online}
Amos Fiat.
\newblock Online algorithms: The state of the art (lecture notes in computer
  science).
\newblock 1998.

\bibitem{gionis1999similarity}
Aristides Gionis, Piotr Indyk, Rajeev Motwani, et~al.
\newblock Similarity search in high dimensions via hashing.
\newblock In {\em Vldb}, volume~99, pages 518--529, 1999.

\bibitem{goemans1995improved}
Michael Goemans and David Williamson.
\newblock Improved approximation algorithms for maximum cut and satisfiability
  problems using semidefinite programming.
\newblock {\em ACM}, 1995.

\bibitem{indyk1998approximate}
Piotr Indyk and Rajeev Motwani.
\newblock Approximate nearest neighbors: towards removing the curse of
  dimensionality.
\newblock In {\em Proceedings of the thirtieth annual ACM symposium on Theory
  of computing}, pages 604--613, 1998.

\bibitem{Kenthapadi_Korolova_Mironov_Mishra_2013}
Krishnaram Kenthapadi, Aleksandra Korolova, Ilya Mironov, and Nina Mishra.
\newblock Privacy via the johnson-lindenstrauss transform.
\newblock {\em Journal of Privacy and Confidentiality}, 5(1), Aug. 2013.

\bibitem{luo2018ace}
Chen Luo and Anshumali Shrivastava.
\newblock Arrays of (locality-sensitive) count estimators (ace): High-speed
  anomaly detection via cache lookups.
\newblock {\em WWW}, 2018.

\bibitem{nakkiran2019more}
Preetum Nakkiran.
\newblock More data can hurt for linear regression: Sample-wise double descent.
\newblock {\em arXiv preprint arXiv:1912.07242}, 2019.

\bibitem{rabkin2014aggregation}
Ariel Rabkin, Matvey Arye, Siddhartha Sen, Vivek~S Pai, and Michael~J Freedman.
\newblock Aggregation and degradation in jetstream: Streaming analytics in the
  wide area.
\newblock In {\em 11th $\{$USENIX$\}$ Symposium on Networked Systems Design and
  Implementation ($\{$NSDI$\}$ 14)}, pages 275--288, 2014.

\bibitem{shrivastava2014mips}
Anshumali Shrivastava and Ping Li.
\newblock Asymmetric lsh (alsh) for sublinear time maximum inner product search
  (mips).
\newblock In {\em NIPS}, 2014.

\bibitem{sohler2011subspace}
Christian Sohler and David~P Woodruff.
\newblock Subspace embeddings for the l1-norm with applications.
\newblock In {\em Proceedings of the forty-third annual ACM symposium on Theory
  of computing}, pages 755--764, 2011.

\end{thebibliography}

\newpage
\section*{Appendix}

We provide proofs for our theorems and further discussion of our classification surrogate losses.

\subsection*{Proof of Theorem~\ref{thm:stormLoss}}
\setcounter{theorem}{0}
\begin{theorem}
The set of STORM-approximable functions $S_L$ contains all LSH collision probabilities and is closed under addition, subtraction, and multiplication.
\end{theorem}
\begin{proof}
It is straightforward to see that STORM can approximate 
$$\sum_{x\in\mathcal{D}}k(x,y)$$
as long as there is an LSH function with the collision probability $k(x,y)$. To prove the theorem, it is sufficient to show that given two LSH collision probabilities $k_1(x,y)$ and $k_2(x,y)$, STORM sketches can approximate the following two functions
$$ f_1(y) = \sum_{x \in \mathcal{D}} k_1(x,y) \pm k_2(x,y) \qquad f_2(y) = \sum_{x\in\mathcal{D}} k_1(x,y)k_2(x,y)$$
Note that one can always write a product of (weighted) sums $\left(\sum_{n}w_n k_n(x,y)\right)\left(\sum_{m}w_m k_m(x,y)\right)$ as the sum of (weighted) products $\sum_{n,m} w_n w_m k_n(x,y) k_m(x,y)$. Therefore, the previous two situations ensure that the set is closed under addition, subtraction and multiplication. 

\textbf{Addition and Subtraction:}
Because of the distributive property of addition, 

$$ \sum_{x \in \mathcal{D}} k_1(x,y) \pm k_2(x,y) = \sum_{x\in\mathcal{D}} k_1(x,y) \pm \sum_{x\in\mathcal{D}} k_2(x,y)$$

One can then construct a STORM sketch $\mathcal{S}_1$ for the $k_1(x,y)$ summation and a second STORM sketch $\mathcal{S}_2$ for $k_2(x,y)$ summation. We can estimate any linear combination of $k_1(x,y)$ and $k_2(x,y)$ by with a weighted sum of $\mathcal{S}_1$ and $\mathcal{S}_2$.

\textbf{Multiplication:} To approximate sums over the product $k_1(x,y)k_2(x,y)$, we rely on LSH hash function compositions. Suppose we have an LSH function $l_1(x)$ with collision probability $k_1(x,y)$ and $l_2(x)$ with $k_2(x,y)$. Consider the hash function $l(x) = \pi(l_1(x),l_2(x))$ where $\pi(\mathbf{a},\mathbf{b})$ is an injective (or unique) mapping from $\mathbb{Z}^2 \to \mathbb{Z}$. An example of such a mapping is the function $\pi(\mathbf{a},\mathbf{b}) = p_1^a p_2^b$ where $p_1$ and $p_2$ are coprime. Since the mapping is injective, this means that $l(x) = l(y)$ only when $l_1(x) = l_1(y)$ and $l_2(x) = l_2(y)$. Therefore, 
$$ \mathrm{Pr}[l(x) = l(y)] = \mathrm{Pr}[l_1(x) = l_1(y)\cap l_2(x) = l_2(y)]$$
Make the choice of the LSH functions $l_1$ and $l_2$ independently, so that the probability factorizes 
$$ = \mathrm{Pr}[l_1(x) = l_1(y)]\mathrm{Pr}[l_2(x) = l_2(y)] = k_1(x,y)k_2(x,y)$$
Therefore, one can construct a STORM sketch for the product using the LSH function $l(x) = \pi(l_1(x),l_2(x))$. 

\end{proof}

\subsection*{Proof of Theorem~\ref{thm:regression}}

\begin{theorem}
\label{thm:regression}
When $p\geq 2$, the PRP collision probability $g([\theta,-1],[\mathbf{x},y])$is a convex surrogate loss for the linear regression objective such that 
$$ \operatorname*{arg\,min}_\theta \sum_{\mathbf{x},y\in\mathcal{D}} g([\theta,-1],[\mathbf{x},y]) =  \operatorname*{arg\,min}_\theta \|\mathbf{y} - \mathbf{X}\theta\| $$
\end{theorem}
\begin{proof}

For the surrogate ERM problem to have the same solution as the linear regression ERM problem, it is sufficient to show two things: that the surrogate loss is convex and that the global minima of the surrogate loss and the linear regresssion loss appear in the same location. The surrogate loss is 
$$ g([\theta,-1],[\mathbf{x},y]) = \frac{1}{2}\left(1 - \frac{1}{\pi}\cos^{-1}({\langle [\theta,-1],[\mathbf{x},y]\rangle})\right)^{p} + \frac{1}{2}\left(1 - \frac{1}{\pi}\cos^{-1}(- {\langle [\theta,-1],[\mathbf{x},y]\rangle})\right)^{p} $$
and the corresponding empirical risk minimization problem is 
$$\theta^\star = \operatorname*{arg\,min}_\theta \sum_{\mathbf{x},y \in \mathcal{D}} g([\theta,-1],[\mathbf{x},y])$$
For the sake of notation, we will put $a = [\theta,-1]$, $b = [\mathbf{x},y]$, and 
$$ f(\mathbf{a},\mathbf{b}) = \left(1 - \frac{1}{\pi}\cos^{-1}({\langle [\theta,-1],[\mathbf{x},y]\rangle})\right)\qquad \mathcal{L}(\mathbf{a},\mathbf{b}) = \frac{1}{2}f(\mathbf{a},\mathbf{b})^p + \frac{1}{2}f(\mathbf{a},-\mathbf{b})^p$$
We will use the fact that 
$$ \nabla_a f(\mathbf{a},\mathbf{b}) = \nabla_a f(\mathbf{a},-\mathbf{b}) = \frac{\mathbf{b}}{\pi\sqrt{1- |\langle \mathbf{a},\mathbf{b} \rangle|^2}}$$

\textbf{Location of Minima:} The minimum of the surrogate loss is same as the minimum for least squares linear regression. Using the chain rule, the gradient of the surrogate loss is
\begin{align*}
\nabla_{a} \mathcal{L}(\mathbf{a},\mathbf{b}) = \frac{p \left( f(\mathbf{a},\mathbf{b})^{p-1} -f(\mathbf{a},-\mathbf{b})^{p-1}\right)}{2\pi \sqrt{1-|\langle \mathbf{a},\mathbf{b}\rangle|^2}} b
\end{align*}
When $p = 1$, the gradient is always zero. When $p\geq 2$, the derivative is zero when $\langle \mathbf{a},\mathbf{b}\rangle = \langle [\theta,-1],[\mathbf{x},y]\rangle = 0$ because that is where $f(\mathbf{a},\mathbf{b}) = f(\mathbf{a},-\mathbf{b})$. Thus, the surrogate loss has the same minimizer as the least squares loss. 

\textbf{Convexity:} At index $(i,j)$, the Hessian of the surrogate loss is 
$$[\nabla^2_{a}\mathcal{L}(\mathbf{a},\mathbf{b}) ]_{(i,j)} = \frac{\partial}{\partial a_j} [\nabla_a \mathcal{L}(\mathbf{a},\mathbf{b})]_i =\frac{\partial}{\partial a_j} b_i  \frac{p \left( f(\mathbf{a},\mathbf{b})^{p-1} -f(\mathbf{a},-\mathbf{b})^{p-1}\right)}{2\pi \sqrt{1-|\langle \mathbf{a},\mathbf{b}\rangle|^2}} $$
Thus, the Hessian
$$ \nabla^2_a \mathcal{L}(\mathbf{a},\mathbf{b}) = b^{\top} \nabla_a \left(\frac{p f(\mathbf{a},\mathbf{b})^{p-1} - p f(\mathbf{a},-\mathbf{b})^{p-1}}{2\pi \sqrt{1 - |\langle \mathbf{a},\mathbf{b}\rangle|^2}}\right)$$
The gradient 
\begin{align*}
&\nabla_a \left(\frac{p f(\mathbf{a},\mathbf{b})^{p-1} - p f(\mathbf{a},-\mathbf{b})^{p-1}}{2\pi \sqrt{1 - |\langle \mathbf{a},\mathbf{b}\rangle|^2}}\right) \\ 
= &\left(\frac{p(p-1)f(\mathbf{a},\mathbf{b})^{p-2}}{2\pi\sqrt{1 - |\langle \mathbf{a},\mathbf{b}\rangle|^2}}\nabla_a f(\mathbf{a},\mathbf{b}) + \frac{p(p-1)f(\mathbf{a},-\mathbf{b})^{p-2}}{2\pi\sqrt{1 - |\langle \mathbf{a},\mathbf{b}\rangle|^2}}\nabla_a f(\mathbf{a},-\mathbf{b})\right)\\
& + \left(p f(\mathbf{a},\mathbf{b})^{p-1} - p f(\mathbf{a},-\mathbf{b})^{p-1}\right) \nabla_a \left(\frac{1}{2\pi \sqrt{1 - |\langle \mathbf{a},\mathbf{b}\rangle|^2}}\right)
\end{align*}
Simplifying, we obtain 
$$ \frac{\mathbf{b}}{2\pi} \left(
\frac{p(p - 1) f(\mathbf{a},\mathbf{b})^{p-2} + p(p-1)f(\mathbf{a},\mathbf{b})^{p-2}}{1 - |\langle \mathbf{a},\mathbf{b} \rangle|^2} + \frac{p f(\mathbf{a},\mathbf{b})^{p-1} - p f(\mathbf{a},-\mathbf{b})^{p-1}}{(1 - |\langle \mathbf{a},\mathbf{b} \rangle |^2)^{3/2}}\right) $$

Which gives the following expression for the Hessian
$$ \nabla^2_{\theta} \mathcal{L}(\mathbf{a},\mathbf{b}) =  \textbf{b}^{T}\textbf{b} \rho (\mathbf{a},\mathbf{b}) $$

where 
$$ \rho (\mathbf{a},\mathbf{b}) = \frac{1}{2\pi} \left(
\frac{p(p - 1) f(\mathbf{a},\mathbf{b})^{p-2} + p(p-1)f(\mathbf{a},\mathbf{b})^{p-2}}{1 - |\langle \mathbf{a},\mathbf{b} \rangle|^2} + \frac{p f(\mathbf{a},\mathbf{b})^{p-1} - p f(\mathbf{a},-\mathbf{b})^{p-1}}{(1 - |\langle \mathbf{a},\mathbf{b} \rangle |^2)^{3/2}}\right) $$

It is easy to see that $\rho (\mathbf{a},\mathbf{b}) >0 \  \ \forall \langle \mathbf{a},\mathbf{b} \rangle \in (-1, 1)$. Hence, the Hessian is positive semidefinite and the function is convex in $\mathbf{a}$. Also note that the function is convex in $\theta$, since the restriction of the last dimension of $\mathbf{a}$ to $-1$ is the restriction of a convex function to a convex set.

\end{proof}

\subsection*{Proof of Theorem~\ref{thm:class}}

\begin{figure*}[t]
\centering
\includegraphics[width=4in,keepaspectratio]{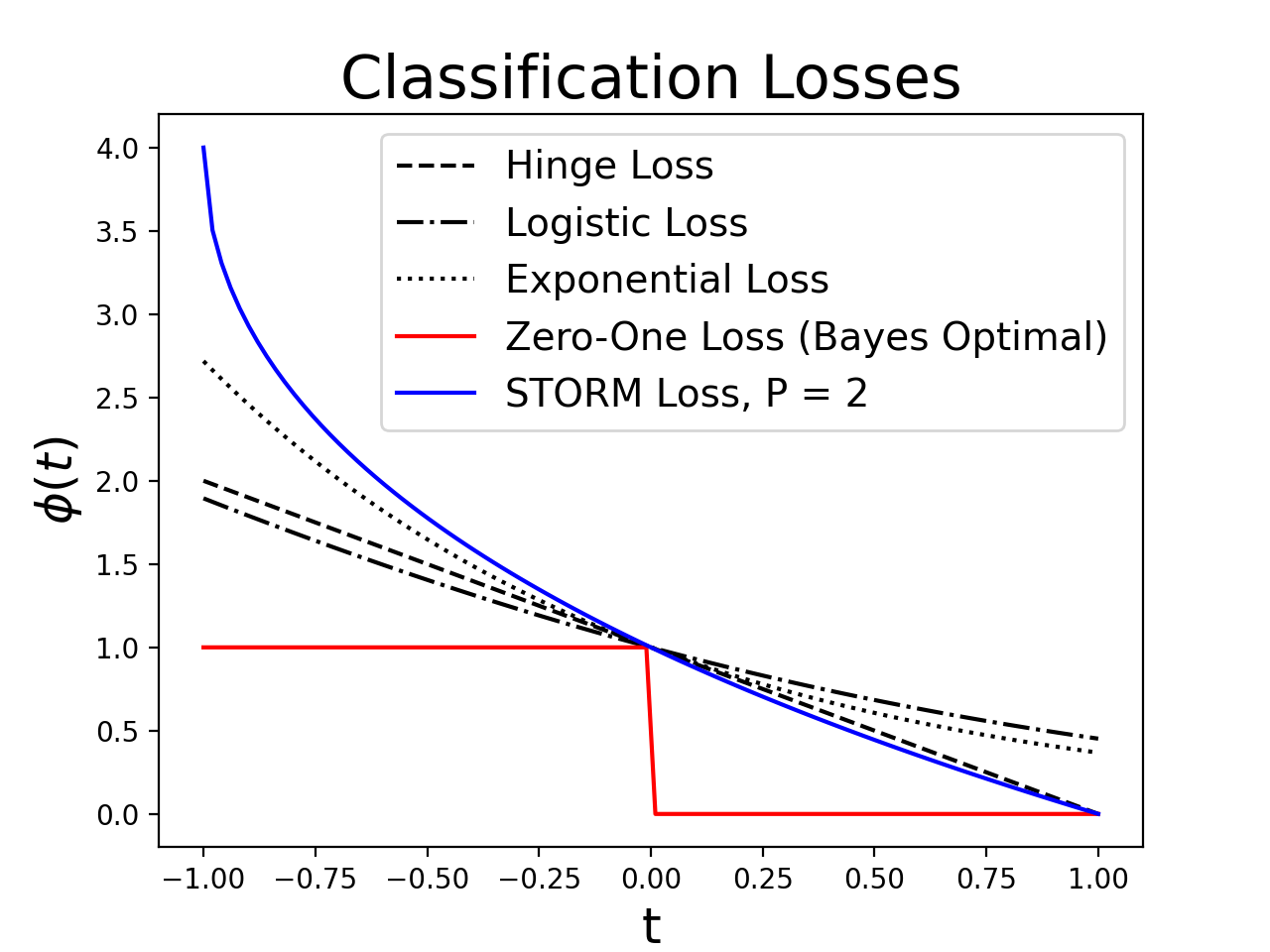}
\caption{Comparison of various classification losses.}
\label{fig:class_losses}
\end{figure*}

\begin{theorem}
\label{thm:class}
Consider labels $y\in\{-1,1\}$. The loss function $g(\theta,[\mathbf{x},y])$ for the linear hyperplane classifier is a classification-calibrated margin loss and is STORM-approximable. 
$$ g(\theta,[\mathbf{x},y]) = 2^p\left(1 - \frac{1}{\pi} \cos^{-1}(-y\langle \theta, \mathbf{x} \rangle) \right)^{p}$$
\end{theorem}
\begin{proof}
First, we show that the loss is classification-calibrated. Then, we show that the loss can be estimated using STORM. 

\textbf{Loss is Classification-Calibrated: } A necessary and sufficient condition for a convex\footnote{For non-convex $\phi(t)$, the sufficient conditions are more complicated.} loss function $\phi(t)$ to be classification-calibrated is for $\frac{d}{dt} \phi(t) < 0$ at $t = 0$. Here, $t = y h(x)$, where $h(x)$ is the model. For a linear hyperplane classifier, $h(x) = \langle \theta,\mathbf{x}\rangle$. The loss is therefore 

$$ \phi(t) = 2^p\left(1 - \frac{1}{\pi}\cos^{-1}(-t)\right)^p$$

$\phi(t)$ is convex when $p \geq 2$ for the same reasons discussed in the proof of Theorem~\ref{thm:regression}. Note that the simple asymmetric LSH for the inner product that we have used throughout the paper\footnote{There are other asymmetric inner product LSH functions without this requirement, and in practice one usually scales the data to make the condition true.} requires $t \in [-1,1]$. The derivative is
$$ \frac{d}{dt} \phi(t) =2^p p\left(1 - \frac{1}{\pi}\cos^{-1}(-t)\right)^{p-1} \frac{-1}{\pi \sqrt{1 - t^2}}$$
At the origin, $\frac{d}{dt} \phi(t)$ is $-1/\pi$. Therefore, the loss is classification calibrated. Figure~\ref{fig:class_losses} compares our STORM surrogate classification loss against popular margin losses.

\textbf{Loss is STORM-Approximable:} Consider the asymmetric LSH function for the inner product where we premultiply $\mathbf{x}$ by $-y$. The collision probability under this LSH function is 
$$ k(\theta,\mathbf{x}) = \left(1 - \frac{1}{\pi}\cos^{-1}(\langle \theta,-y\mathbf{x}\rangle)\right)$$
as desired. 
\end{proof}

\end{document}